\documentclass[letterpaper,10pt,conference]{ieeeconf} 
\IEEEoverridecommandlockouts
\makeatletter
\@ifundefined{proof}{}{%

}
\makeatother
\usepackage{amsmath,amssymb,mathtools,amsthm,bm}

\theoremstyle{plain}
\newtheorem{theorem}{Theorem}
\newtheorem{lemma}[theorem]{Lemma}

\theoremstyle{definition}
\newtheorem{assumption}[theorem]{Assumption}

\theoremstyle{remark}

\usepackage{graphicx}          
\usepackage{subcaption}        
\usepackage{booktabs}          
\usepackage{multirow}
\usepackage{nicematrix}
\usepackage{tabularx}
\usepackage[table]{xcolor}
\usepackage{hyperref}
\usepackage{marvosym}

\usepackage{algorithm}
\usepackage{algorithmic}

\usepackage{tcolorbox}

\usepackage{xspace}            

\newcommand{\ours}[0]{{SiLRI}\xspace}

\title{\LARGE \bf
Real-world Reinforcement Learning from Suboptimal Interventions
}

\begin{document}

\author{
Yinuo Zhao$^{1,2}$,
Huiqian Jin$^{1,3}$,
Lechun Jiang$^{1,3}$,
Xinyi Zhang$^{1,4}$,
Kun Wu$^{1}$,
Pei Ren$^{1}$\\
Zhiyuan Xu$^{1\dagger}$,
Zhengping Che$^{1\dagger}$,
Lei Sun$^{3}$,
Dapeng Wu$^{2}$,
Chi Harold Liu$^{4}$,
Jian Tang$^{1}$\textsuperscript{\Letter}\\
{\small
$^{1}$Beijing Innovation Center of Humanoid Robotics,
$^{2}$City University of Hong Kong}\\
{\small$^{3}$Nankai University,
$^{4}$Beijing Institute of Technology}\\
}

\maketitle
\begingroup
\renewcommand{\thefootnote}{}
\footnotetext{\Letter~ Corresponding author: Jian Tang (jian.tang@x-humanoid.com). $\dagger$~ Project leaders: Zhengping Che and Zhiyuan Xu. This work is done during Yinuo Zhao's internship at Beijing Innovation Center of Humanoid Robotics.}
\endgroup
\thispagestyle{empty}
\pagestyle{empty}

\begin{abstract}
Real-world reinforcement learning (RL) offers a promising approach to training precise and dexterous robotic manipulation policies in an online manner, enabling robots to learn from their own experience while gradually reducing human labor. 
However, prior real-world RL methods often assume that human interventions are optimal across the entire state space, overlooking the fact that even expert operators cannot consistently provide optimal actions in all states or completely avoid mistakes. Indiscriminately mixing intervention data with robot-collected data inherits the sample inefficiency of RL, while purely imitating intervention data can ultimately degrade the final performance achievable by RL.
The question of \textit{how to leverage potentially suboptimal and noisy human interventions to accelerate learning without being constrained by them} thus remains open. 
To address this challenge, we propose \textbf{\ours}, a state-wise Lagrangian reinforcement learning algorithm for real-world robot manipulation tasks. Specifically, we formulate the online manipulation problem as a constrained RL optimization, where the constraint bound at each state is determined by the uncertainty of human interventions. We then introduce a state-wise Lagrange multiplier and solve the problem via a min-max optimization, jointly optimizing the policy and the Lagrange multiplier to reach a saddle point.
Built upon a human-as-copilot teleoperation system, our algorithm is evaluated through real-world experiments on diverse manipulation tasks. Experimental results show that \ours effectively exploits human suboptimal interventions, reducing the time required to reach a 90\% success rate by at least 50\% compared with the state-of-the-art RL method HIL-SERL, and achieving a 100\% success rate on long-horizon manipulation tasks where other RL methods struggle to succeed. Project website:~\href{https://silri-rl.github.io/}{https://silri-rl.github.io/}

\end{abstract}

\section{Introduction}

Vision-based robotic manipulation has recently achieved remarkable progress in real-world applications. Either Transformer-based methods (e.g., ACT~\cite{zhao2023learning}) or diffusion-based approaches (e.g., DP~\cite{chi2023diffusion}) achieves strong performance by leveraging expert demonstrations collected through teleoperation~\cite{xu2025hacts,wu2025robocopilot,wu2024gello}. Despite their success, Imitation Learning (IL) paradigms remain constrained by their heavy dependence on large amounts of high-quality demonstrations, which require substantial human effort. In contrast, reinforcement learning (RL) offers a more autonomous alternative by enabling robots to explore and optimize policies in their own way, thereby reducing dependence on human supervision. However, unlike simulation-based RL, where agents can explore freely and at low cost, practical deployment in complex real-world scenarios is still limited by RL’s poor sample efficiency and safety concerns.

\begin{figure}[tb]
    \centering
    \subfloat{\includegraphics[width=1.0\columnwidth]{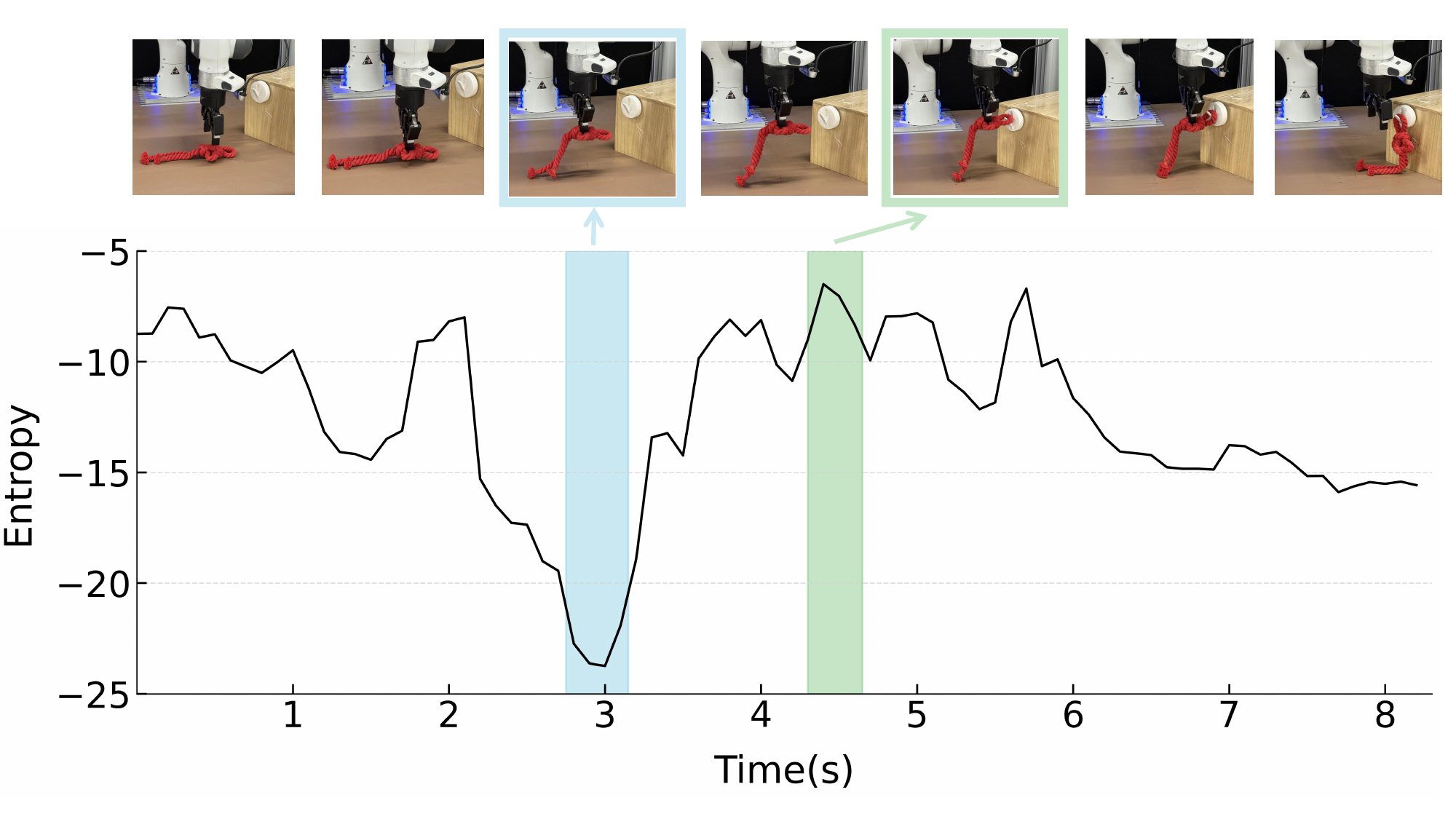}}
    \caption{\textbf{Entropy of Human Interventions Across States.} Entropy is estimated with a multivariate normal distribution model. In low-entropy states (blue region), human operators intervene consistently and confidently, whereas in high-entropy states (green region), their interventions are inconsistent, indicating uncertainty.}
    \label{fig:entropy}
\end{figure}

\begin{figure*}[tb]
    \centering
    \subfloat{\includegraphics[width=1.0\textwidth]{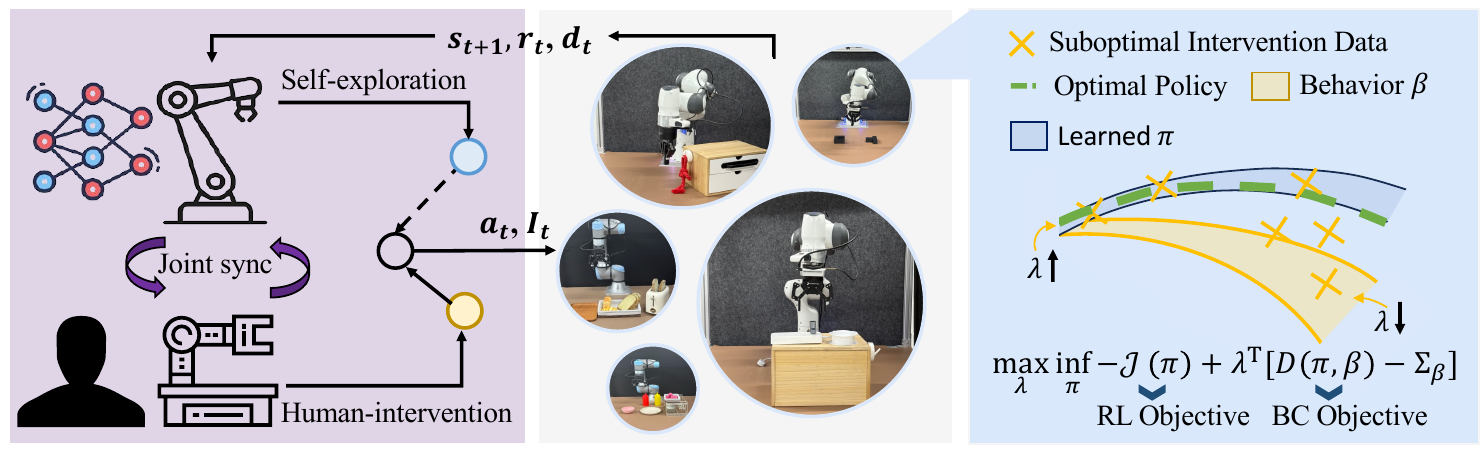}}
    \caption{\textbf{\ours Enables Effective Real-world RL from Suboptimal Interventions.} Left: A human-as-copilot teleoperation system that enables seamless human intervention. Right: The overall optimization objective of \ours. Relaxing the constraint in high-entropy states enables the learned policy $\pi$ to converge to a policy that outperforms human behavior policy $\beta$.}
    \label{fig:overview}
\end{figure*}

Human-in-the-loop Reinforcement Learning (HIL-RL)~\cite{luo2025precise,chen2025conrft} has recently emerged as a promising paradigm for tackling complex real-world robot manipulation tasks, achieving near human-level performance with highly dexterous skills. HIL-RL typically leverages three types of data: online policy transitions, online human interventions, and offline human demonstrations. For clarity, we group the first two as online data and the latter two as intervention data. In practice, due to the high-dimensional exploration space and sparse reward signals in real-world environments, existing HIL-RL methods typically store online and intervention data in separate buffers (with online human interventions duplicated) and sample from them uniformly during RL updates. This strategy increases the sampling probability of scarce yet informative successful transitions from intervention data, which is a key factor underlying their current success. However, intervention data are often integrated into the RL pipeline without careful consideration of their underlying properties. We illustrate this issue with the following example.

As shown in Fig.~\ref{fig:entropy}, we plot the policy entropy along a human-collected trajectory in the Hang Chinese Knot task. The entropy is estimated using a multivariate normal distribution fitted to 20 trajectories of human demonstrations. We observe that in low-entropy states, such as when picking up the Chinese knot, humans tend to provide consistent actions (i.e., similar actions under the same state). In contrast, in states that require high precision, humans fail to produce consistent actions, leading to high entropy. This suboptimality reveals a fundamental challenge: different states require different balances between the RL and IL objectives. In low-entropy states, imitation is effective, but in high-entropy states, suboptimal interventions can hinder performance, making the RL objective a more reliable guide.

To address this challenge, we propose \textbf{\ours}, a \underline{S}tate-w\underline{i}se \underline{L}agrangian \underline{R}einforcement learning framework from suboptimal \underline{I}nterventions for real-world robot manipulation training. Specifically, we first formulate a constrained RL optimization problem, where the constraint bound at each state is determined by the corresponding human uncertainty. As shown in the right part of Fig.~\ref{fig:overview}, in states where humans provide consistent data (low entropy), the learned policy is constrained to stay close to the human behavior policy through increasing Lagrange multiplier $\lambda$, to improve training efficiency. In high-entropy states, the constraint is relaxed, encouraging the policy to be optimized primarily through its own estimated critic, i.e., RL objective. 
To realize this behavior, we introduce a state-wise Lagrange multiplier in the imitation term and formulate the problem as a min-max optimization. The policy objective is minimized while the dual objective over the Lagrange multiplier is maximized, enabling the state-wise multiplier to be adaptively adjusted to balance the RL and IL objectives. This design improves RL sample efficiency while ensuring convergence to a policy that can surpass suboptimal intervention data.

Building \ours on a human-as-copilot teleoperation system, as shown in the left part of Fig.~\ref{fig:overview}, we conduct extensive real-world experiments on two robotic embodiments. Comparisons against a classical online IL method and state-of-the-art (SOTA) real-world RL methods demonstrate the effectiveness of \ours.

We summarize our contributions as follows:
\begin{itemize}
    \item We formulate a constrained RL problem that explicitly accounts for human uncertainty across states, enabling more effective use of suboptimal human interventions.
    \item To solve this constrained RL problem, we introduce learnable state-wise Lagrange multipliers and cast the optimization as a min-max problem, allowing the agent to adaptively trade off between RL and IL objectives.
    \item Built on a human-robot collaboration system, our real-world experiments demonstrate that \ours reaches a 90\% success rate at least 50\% faster (corresponding to an 18-minute reduction) than SOTA method HIL-SERL, while achieving a 100\% success rate on long-horizon manipulation tasks where other RL methods struggle.
\end{itemize}

\section{Related Works}
\paragraph*{Imitation Learning for Real-world Manipulations}
Imitation learning has dominated the end-to-end learning paradigm~\cite{hao2021matters,albanie2020end} for real-world manipulation for decades. With high-quality demonstrations collected via teleoperation systems (e.g., VR headsets~\cite{khazatsky2024droid}, homologous bilateral system~\cite{xu2025hacts,wu2025robocopilot}), it has served as a simple, efficient paradigm for vision-based robot manipulation. Recent methods, including ACT~\cite{zhao2023learning} and DP~\cite{chi2023diffusion}, further improve performance through more expressive models. More recently, Vision-Language-Action (VLA) models, such as Octo~\cite{team2024octo}, OpenVLA~\cite{kim2024openvla} and $\pi_0$~\cite{black2410pi0}, have shown promise for general manipulation task learning by leveraging large parameter counts and Internet-scale pretraining. However, their substantial computational cost makes online fine-tuning or even learning of complex, dexterous manipulation skills challenging. 
In this work, we focus on developing algorithms that can efficiently leverage suboptimal human interventions, while using simple model architectures mainly borrowed from prior work~\cite{luo2025precise}.

\paragraph*{Reinforcement Learning for Real-world Manipulations}
Recent works~\cite{liu2025can,lu2025vla,chen2025rlrc,julg2025refined} explore reinforcement learning algorithm to unlock the generalization ability of large models. Most approaches build on PPO~\cite{schulman2017proximal} with an OpenVLA~\cite{kim2024openvla} backbone and experiment on simulations. While promising, fully autonomous RL exploration without human intervention can be unsafe and inefficient for real-world training. For real-world robot manipulation tasks, RL methods are broadly divided into two categories based on how they incorporate human assistance. The first category directly uses human intervention actions. For example, ConRFT~\cite{chen2025conrft} leverages the human intervention pipeline from prior work HIL-SERL~\cite{luo2025precise}, combining behavior cloning and Q-learning in a consistency policy for both offline and online training. iRe-VLA~\cite{guo2025improving} alternates between supervised learning (SL) and reinforcement learning (RL) for stable fine-tuning. In the SL stage, the VLM backbone and action head are trained on successful trajectories using LoRA~\cite{hu2022lora}. In the RL stage, only lightweight action and critic heads are fine-tuned for efficiency. These works typically treat human interventions as optimal guidance and incorporate the corresponding data via an imitation loss term. The second category treats human intervention as preference signals. HAPO~\cite{xia2025robotic}, for instance, integrates intervention trajectories into an action preference optimization process within an RLHF framework, helping VLAs avoid failure actions and adapt corrective behaviors. Other methods, such as MILE~\cite{korkmaz2025mile} and RLIF~\cite{luorlif}, also use intervention signals rather than direct actions, but with small visuomotor backbones. These works remain limited to simulation environments or simple real-world tasks. In this paper, we address more complex and general robotic manipulation tasks, and our approach falls into the first category that directly leverages human intervention action but does not assume optimal human operation.

\paragraph*{Lagrange multipliers}
Lagrange multipliers have been widely adopted in RL, but most methods do not explicitly optimize their exact values, relying instead on heuristic or fixed estimates. For example, in the classical PPO algorithm~\cite{schulman2017proximal}, a trust-region style constraint is imposed on the distance between the updated policy and the old policy to avoid policy collapse. Due to the unbounded nature of both the Lagrange multiplier and the KL divergence, PPO replaces explicit multiplier optimization with a clipped surrogate objective. In offline RL, constrained optimization is used to prevent bootstrapping on out-of-distribution actions by constraining the learned policy to remain close to the behavior policy that generated the dataset. In AWAC~\cite{nair2020awac}, for instance, the optimization objective is simplified to a forward KL form with a heuristic temperature (multiplier) parameter. More recently, several works directly optimize Lagrange multipliers. PPO-Lagrangian extends PPO to a constrained RL setting, enabling the algorithm to handle explicit constraints, while other methods~\cite{seo2025state} introduce state-dependent Lagrange multipliers so that the policy can account for state-wise safety constraints. Our work builds on these studies by introducing state-wise Lagrange multipliers into human-in-the-loop RL.

\section{Preliminaries}\label{sec:pre}
In this section, we first formulate the decision-making process as a Markov Decision Process (MDP). We then introduce a constrained optimization problem under a simplified constraint formulation, which more accurately captures the objective of online RL under human assistance. 

Following~\cite{luo2025precise}, we model the MDP in robotic manipulation as a tuple $(\mathcal{S}, \mathcal{A}, \mathcal{I}, \rho, \mathcal{P}, \mathcal{R}, \gamma)$. The state $\bm{s}\in\mathcal{S}$ comprises multiple RGB images together with the robot’s proprioceptive signals. The action $a\in\mathcal{A}$ (e.g., a desired end-effector 6D movement) is given either by a learned policy $\pi$ or by a human via a teleoperation interface. An intervention indicator $I(s,a)\in\mathcal{I}$ marks the action’s source, with $I(s,a)=1$ for human intervention and $I(s,a)=0$ for autonomous control. $\rho(\bm{s}_0)$ denotes the initial-state distribution; $\mathcal{P}$ denotes the environment’s state-transition dynamics; and the reward function $\mathcal{R}$ maps $\mathcal{S}\times\mathcal{A}$ to scalar rewards. Finally, $\gamma$ is the discount factor that trades off immediate and future returns, with larger $\gamma$ placing more emphasis on long-term returns.

In real-world RL, reward signals are typically sparse: the agent receives positive feedback only upon successful task completion. In this work, we follow prior approaches~\cite{luo2025precise,chen2025conrft} and adopt a sparse reward formulation. A reward of 10 is provided only upon successful task completion, while a penalty of -0.05 is applied at every other step. 

Previous RL methods, such as ConRFT~\cite{chen2025conrft} and HIL-SERL~\cite{luo2025precise}, simply treated the learning process as an typical unconstraint optimization problem and the objective $J(\pi)$ is defined as:
\begin{equation}\label{eqn:RL_objective}
    J(\pi) = \mathbb{E}_{\tau \sim \pi}\left[\sum_{t=0}^T\gamma^tr(\bm{s}_t, \bm{a}_t) \right].
\end{equation}

Although the above optimization objective is compatible with suboptimal human intervention data, it overlooks the potential optimality embedded in such data. In states where humans can provide reliable interventions, simple behavior cloning offers more effective guidance than a Q-network. Conversely, in other states, a well-learned Q-network can guide the policy beyond suboptimal human behavior, enabling performance that surpasses human capability. Identifying these states and automatically balancing the BC and RL objectives is therefore crucial for efficient real-world training. In the following section, we present a constrained optimization formulation and its solution to enable efficient real-world robot manipulation learning through more effective utilization of human interventions.

\section{Proposed Method: \ours}
In this section, we present \ours, \underline{S}tate-w\underline{i}se \underline{L}agrangian \underline{R}einforcement learning from suboptimal \underline{I}nterventions, for real-world robot manipulation training. We begin by formulating a constrained optimization problem and introducing how state-wise Lagrange multipliers are used to solve it. Next, we detail the loss functions and network architectures employed in \ours to enable efficient and stable learning. We then summarize the overall HIL training paradigm that supports interactive policy refinement. Finally, we discuss key design choices and parameter settings necessary for successful deployment in real-world training scenarios.

\subsection{Problem Formulation}
Recall the unconstrained objective in Eqn.~(\ref{eqn:RL_objective}). Directly optimizing the policy in the direction of increasing Q-values can be unreliable due to inaccurate value estimation, especially during the early stages of training. To address this challenge, we reformulate the objective as a constrained optimization problem as follows:
\begin{equation}
\label{eqn:original_objective}
\begin{aligned}
\max_{\pi} \quad & J(\pi) \\
\text{s.t.} \quad 
& \mathrm{D}\big(\pi(\cdot| \bm{s})\Vert\beta(\cdot|\bm{s})\big)
   \le \epsilon, \qquad \forall \bm{s},
\end{aligned}
\end{equation}
where $\pi$ denotes the deterministic robot policy, and $\beta$ represents the stochastic reference policy derived from human interventions. Although the optimal manipulation policy for the robot is assumed to be deterministic, human operators inherently exhibit variability and are unable to consistently reproduce continuous control actions. As a result, we model the human reference policy $\beta$ as stochastic. The constraint is designed to prevent the learned policy $\pi$ from deviating excessively from $\beta$, which could otherwise lead to unrecoverable or unsafe behaviors.

However, two key challenges arise when optimizing the above objective. First, measuring the distance between a deterministic policy $\pi$ and a stochastic policy $\beta$ is nontrivial. Second, even if we treat $\pi$ as a stochastic policy and evaluate the Kullback-Leibler (KL) divergence between these two, empirical optimization remains difficult due to the unbounded nature of the KL divergence. To address these issues, we reformulate the constraint as follows:
\begin{equation}
\label{eqn:objective}
\begin{aligned}
\max_{\pi} \quad & J(\pi) \\
\text{s.t.} \quad 
& \parallel \mu(\pi(\cdot|\bm{s})) - \mu(\beta(\cdot|\bm{s}))\parallel 
   \le \kappa \cdot \sigma_{\beta}(\bm{s}), \qquad \forall \bm{s},
\end{aligned}
\end{equation}
where $\mu$ denotes the output of $\pi$ and the mean of $\beta$, $\sigma_{\beta}$ is the standard deviation of policy $\beta$, and $\kappa$ is a constant value that adjust the tightness of the constraint. In Appendix~\ref{sec:proof}, we proof that under mild assumptions, the constraint condition in Eqn.~(\ref{eqn:original_objective}) can be simplified into the constraint condition in Eqn.~(\ref{eqn:objective}). For simplicity, in the following we omit the symbol $\mu$ and use $\pi$ and $\beta$ to denote the outputs (i.e., means) of the corresponding policies.

The constraint in Eqn.~(\ref{eqn:objective}) enforces that the optimized policy $\pi$ should remain close to the reference policy $\beta$ in states where humans provide consistent intervention data. In contrast, for states where human behavior is more uncertain (i.e., where the reference policy exhibits high variance), the policy $\pi$ is allowed to deviate and explore more freely, optimizing through RL under a relaxed constraint.

\begin{figure}
    \centering
    \includegraphics[width=0.8\columnwidth]{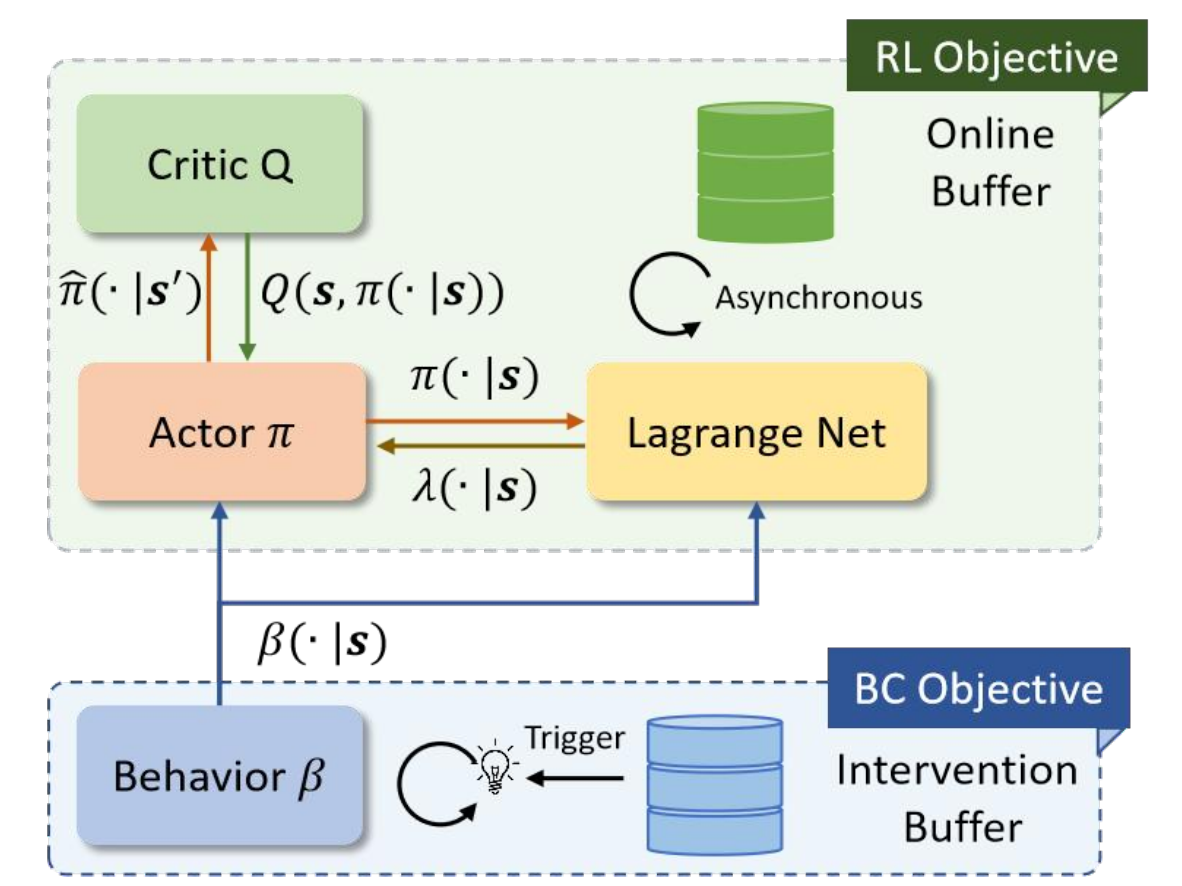}
    \caption{\textbf{Network Components in \ours.} The networks $Q$, $\pi$, and $\lambda$ are updated asynchronously during data collection using online buffer, while the network $\beta$ is updated periodically after a fixed number of new samples have been added to the intervention buffer.
    \label{fig:network}}
\end{figure}

\subsection{State-wise Lagrange Multipliers}
A common approach to solving constrained optimization problems is to introduce a Lagrange multiplier $\lambda \geq 0$ that penalizes constraint violations. For the optimization problem in Eqn.~(\ref{eqn:objective}), the corresponding Lagrangian is
\begin{equation}\label{eqn:dual_problem}
    \mathcal{L}(\pi, \lambda) = -J(\pi) + \lambda^{\top} \left[\mathrm{D}(\pi , \beta) - \kappa\Sigma_{\beta} \right],
\end{equation}
where $\lambda$ is a state-wise Lagrange multiplier, i.e., a nonnegative function over states rather than a scalar value. $\mathrm{D}(\pi , \beta)$ represents the mean deviation in Eqn.~(\ref{eqn:objective}) and $\Sigma_{\beta}$ denotes the standard deviation of $\beta$. The associated Lagrange dual problem can be written as
\begin{equation}
    \sup_{\lambda} \inf_{\pi} \mathcal{L}(\pi, \lambda),
    \quad \text{s.t.}~\lambda(\cdot) \geq 0.
\end{equation}
The objective is to find a saddle point $(\pi^*, \lambda^*)$ that satisfies
\begin{equation}
    \mathcal{L}(\pi, \lambda^*) \geq \mathcal{L}(\pi^*, \lambda^*) \geq \mathcal{L}(\pi^*, \lambda)
\end{equation}
for all admissible $\pi$ and $\lambda$. Specifically, for the policy $\pi$ and a fixed $\lambda$, the optimization problem becomes
\begin{equation}\label{eqn:optim_actor}
    \inf_{\pi} \left[ - J(\pi) + \lambda^{\top} \mathrm{D}(\pi, \beta) \right],
\end{equation}
since $\Sigma_{\beta}$ does not depend on $\pi$. Conversely, for the Lagrange multiplier $\lambda$ and a fixed $\pi$, the optimization problem is
\begin{equation}
    \sup_{\lambda} \lambda^{\top}\left[ \mathrm{D}(\pi, \beta) - \kappa \Sigma_{\beta}\right],
    \quad \text{s.t.}~\lambda(\cdot) \geq 0.
\end{equation}

Since finding the global saddle point is computationally intractable, we adopt gradient-based updates with function approximators. However, there are two challenges related to the behavior policy $\beta$ when solving the above objectives. First, the true policy distribution of the human operator is unavailable. Second, human intervention data is unavailable for states visited exclusively by the learned policy $\pi$. Therefore, we cannot directly measure the distance between $\pi$ and $\beta$ without sampling from the behavior policy. To address these challenges, we approximate $\beta$ with a multivariate Gaussian distribution fitted on the human intervention data. In the following, we describe the network architectures and loss functions used to approximate the optimal solution.

\subsection{Network Architectures and Loss Functions}
As shown in Fig.~\ref{fig:network}, there are four networks to approximate, they are the actor network $\pi$, behavior policy $\beta$, Lagrange network $\lambda$ and critic network $Q$. To improve training stability, we adopt a target actor network and a double $Q$-network architecture, following common practice in off-policy RL.

\paragraph*{Critic optimization}
We follow the standard Bellman update and optimize the critic $Q$ using a temporal-difference objective:
\begin{equation}
\begin{aligned}
\mathcal{L}(\theta^Q) 
    = \mathbb{E}_{\bm{s}, \bm{a}, \bm{s}'} \Bigl[
        \bigl( r & + \gamma \hat{Q}(\bm{s}', \hat{\pi}( \bm{s}')) \\
    &- \min_{k=1,2} Q_k(\bm{s}, \bm{a}; \theta^Q) \bigr)^2
    \Bigr].
\end{aligned}
\label{eqn:q_loss}
\end{equation}
where $\theta^Q$ denotes the parameters of the critic $Q$, and $\hat{\pi}$ and $\hat{Q}$ are the target policy and target critic, respectively. Both target networks are updated via Polyak averaging~\cite{van2016deep} from the current networks.

\begin{algorithm}[t!]
\caption{\ours~(Learner)} \label{algo:ours}
\begin{algorithmic}[1]
\footnotesize
\STATE Create empty intervention buffer $\mathcal{D}_I$ and Online RL buffer $\mathcal{D}_R$.
\STATE Fill $\mathcal{D}_I$ with small-scale human demonstrations.
\STATE Create RL agent with critic network $Q$, actor network $\pi$, Lagrange network $\lambda$ and behavior network $\beta$, with trainable parameters $\theta^Q$, $\theta^\pi$, $\theta^\lambda$, $\theta^\beta$, respectively. 
\STATE Train $\beta$ with data from $\mathcal{D}_I$ with Eqn.~(\ref{eqn:beta_loss}).
\STATE Initialize $n=0$
\WHILE{LISTEN}
    \IF{Receive transitions from Actor}
    \FOR{$t$ in 1 $\dots$ $T$}
    \STATE Insert transitions into $\mathcal{D}_R$ and insert transitions ($\bm{s}_t$, $\bm{a}_t$, $I_t=1$, $r_t$, $d_t$) into $D_I$.
    \STATE $n = n+1$.
    \IF{$n\bmod 50 = 0$}
        \FOR{ $k$ in 1 $\cdots$ 50}
        \STATE Minimize $\mathcal{L}(\theta^\beta)$ in Eqn.~(\ref{eqn:beta_loss}) with data from $\mathcal{D}_I$.
        \ENDFOR 
    \ENDIF 
    \ENDFOR
    \ENDIF
    \STATE Sample a minibatch of data from $\mathcal{D}_R$ and $\mathcal{D}_I$.
    \STATE Minimize $\mathcal{L}(\theta^Q)$ in Eqn.~(\ref{eqn:q_loss}).
    \STATE Minimize $\mathcal{L}(\theta^\pi)$ in Eqn.~(\ref{eqn:actor_loss}).
    \STATE Minimize $\mathcal{L}(\theta^\lambda)$ in Eqn.~(\ref{eqn:lag_loss}).
    \IF{Send Parameter Period}
    \STATE Send Parameters to the Actor.
    \ENDIF
\ENDWHILE
\label{algo}
\end{algorithmic}
\end{algorithm}

\paragraph*{Actor optimization} The actor network is modeled as a multivariate Gaussian policy whose mean is passed through a Tanh activation to constrain the actions to $(-1, 1)$. During inference, when we need to sample actions from the actor, we add a constant standard deviation of $0.05$ to encourage sufficient exploration and apply clipping to avoid out-of-range samples. During training, recalling the constrained optimization problem in Eqn.~(\ref{eqn:optim_actor}), the actor policy $\pi$ is optimized with a joint objective that combines reinforcement learning (RL) and behavior cloning (BC), balanced by the self-adapting coefficient $\lambda$. The actor loss is defined as
\begin{equation}
\begin{aligned}
\mathcal{L}(\theta^{\pi}) =
\mathbb{E}_{\bm{s}}
\frac{1}{\lambda(\bm{s})+1}
\Bigl[
&- Q(\bm{s}, \pi(\bm{s}; \theta^{\pi})) \\
&+ \lambda(\bm{s})
\bigl\lVert \pi(\bm{s};\theta^\pi) - \beta(\bm{s}) \bigr\rVert_2^2
\Bigr],
\end{aligned}
\label{eqn:actor_loss}
\end{equation}
where $\lVert\cdot\rVert$ denotes the Euclidean norm over the action space, and $\lambda (\bm{s})$ is the state-wise Lagrange multiplier. Following prior work~\cite{seo2025state}, the division $\lambda(\bm{s})+1$ acts as a regularizer to mitigate instability caused by large multiplier values. Note that when $\lambda(\bm{s}) \rightarrow \infty$, the RL term is suppressed and the loss is dominated by the BC term, effectively forcing the policy to follow the behavior policy. Conversely, when $\lambda(\bm{s}) \rightarrow 0$ and the constraint is satisfied, the RL term dominates, allowing the learned policy $\pi$ to improve beyond the behavior policy under the RL objective.

\paragraph*{Lagrange Multiplier Optimization}
The Lagrange multiplier network takes the current state as input and outputs a scalar value that modulates the weight of the constraint term, as defined in Eqn.~(\ref{eqn:dual_problem}). To ensure that the multiplier remains non-negative, a Softplus activation function is applied at the network's output layer. The loss function used to train the Lagrange multiplier is given by:
\begin{equation}
\label{eqn:lag_loss}
\mathcal{L}(\theta^{\lambda}) = \mathbb{E}{\bm{s}} \left[-\lambda(\bm{s};\theta^{\lambda}) \left(\mathrm{D}(\pi, \beta) - \kappa \cdot \sigma_{\beta} - c) \right)\right],
\end{equation}
where $\mathrm{D}(\pi, \beta) = \lVert \pi(\bm{s}; \theta^{\pi}) - \beta(\bm{s}) \rVert_2^2$ denotes the squared distance between the mean actions of the learned policy $\pi$ and the reference policy $\beta$, consistent with that used in the actor optimization. Here, $\sigma_{\beta}$ represents the average standard deviation over the continuous action dimensions of the multivariate Gaussian reference policy, and $\kappa$ denotes the action dimensionality (i.e., $\kappa = 6$). Empirically, we add a constant $c = 0.1$ to relax the constraint. Intuitively, the Lagrange multiplier increases when the deviation between $\pi$ and $\beta$ exceeds the tolerance margin defined by $\kappa \cdot \sigma_{\beta} + c$, and gradually decays toward zero once the constraint is satisfied.

\paragraph*{Behavior Policy Optimization}
As shown in Fig.~\ref{fig:network}, the behavior policy $\beta$ is optimized separately from the other three networks and updated at a lower frequency to maintain stable decision entropy (i.e., the standard deviation of the action distribution). Without such regulation, the entropy may decay to an extremely low value during training, which can hinder learning and reduce policy diversity. To address this, we update the behavior policy only when a fixed amount of new data is added to the replay buffer. The optimization objective of the behavior policy is given by:
\begin{equation}
\label{eqn:beta_loss}
\mathcal{L}(\theta^\beta) = \mathbb{E}_{(\bm{s}, \bm{a})\sim\mathcal{D}_I} \left[ -\log \beta(\bm{a} \mid \bm{s}; \theta^\beta) \right],
\end{equation}
where the state-action pairs $(\bm{s}, \bm{a})$ are sampled from the intervention buffer $\mathcal{D}_I$. This objective corresponds to standard behavior cloning and encourages $\beta$ to maximize the likelihood of the observed human intervention actions.

\subsection{Human-in-the-loop Training Paradigm}
\ours~builds on the asynchronous actor-critic architecture implemented in prior work~\cite{luo2025precise} to increase the update-to-data ratio. Before training, the human operator collects 20 trajectories of demonstrations for each tasks, and stored in intervention buffer $\mathcal{D}_I$.  During training,
at each actor timestep $t$, \ours~ stores a transition
$(\bm{s}_t,\bm{a}_t,I_t,r_t,d_t,\bm{s}_{t+1})$,
where $I_t\in\{0,1\}$ indicates whether the action came from human intervention ($1$) or RL exploration ($0$), and $d_t\in\{0,1\}$ flags task success, which is given by a reward classifier. 
After each episode, the actor streams transitions to the learner and receives the actor policy parameters from the learner. 

Algorithm~\ref{algo:ours} summarizes the learner-side training procedure. Before entering the main training loop, the behavior policy is pre-trained on data from $\mathcal{D}_I$ for 500 steps. During training, upon receiving online data from the actor side, all intervention samples are added to the intervention buffer $\mathcal{D}_I$, and all online samples, including both intervention and non-intervention data, are added to the online buffer $\mathcal{D}_R$. The behavior policy is updated whenever 50 new samples have been added into $\mathcal{D}_I$. After updating the buffers, we sample mini-batches of equal size 128 from $\mathcal{D}_I$ and $\mathcal{D}_R$, and then update the parameters $\theta^Q$, $\theta^\lambda$, and $\theta^\pi$ according to Eqn.~(\ref{eqn:q_loss})-(\ref{eqn:lag_loss}). The learner thread is terminated by the actor once the time limit is reached.

\begin{figure*}[tb]
    \centering
\subfloat{\includegraphics[width=1.0\textwidth]{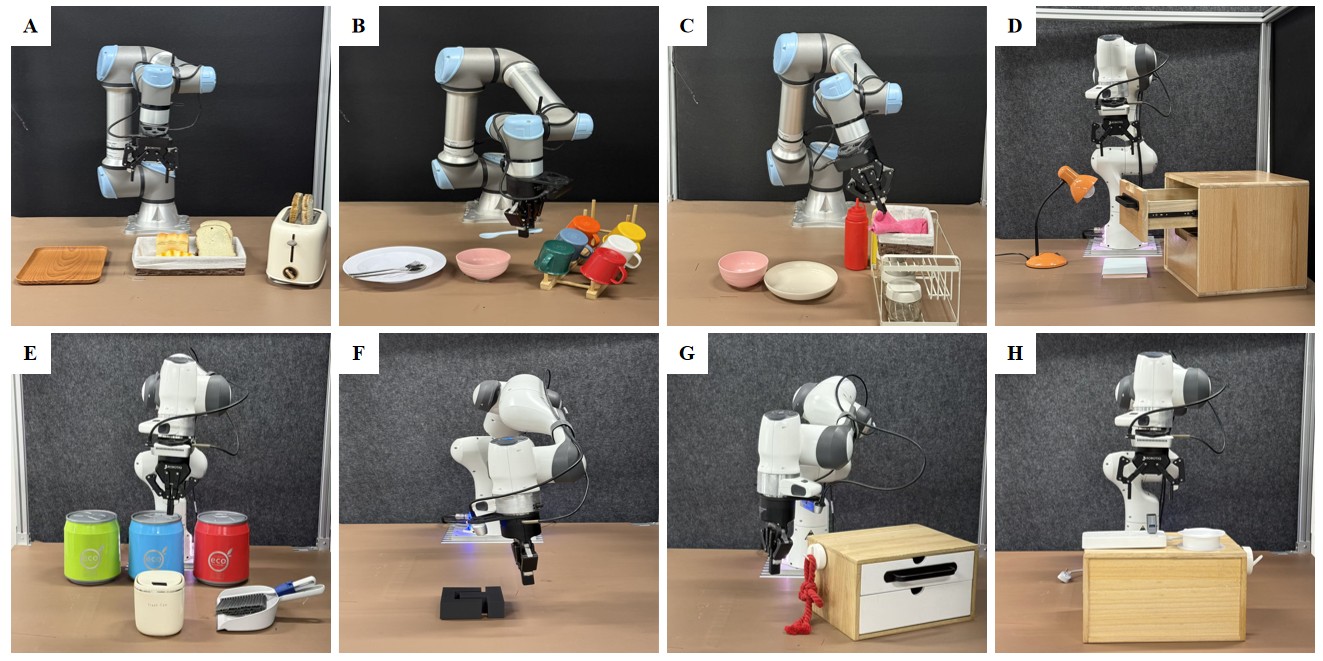}}
    \caption{\textbf{Eight Real-world Manipulation Tasks on Two Embodiments.} (A) Pick-Place Bread (B) Pick-up Spoon (C) Fold Rag (D) Open Cabinet (E) Close Trashbin (F) Push-T (G) Hang Chinese Knot (H) Insert USB.}
    \label{fig:task}
\end{figure*}

\subsection{Key Design Choices and Parameters}
Unlike ideal simulation environments, training real-world RL policies must cope with complex scenarios and imperfect estimations that can significantly affect overall training efficiency. Here, we highlight three key design choices that are critical for the successful training of all real-world RL methods considered in our work.

\paragraph*{Intervention Standard Operating Procedure}
Empirically, we find that the timing of interventions is critical to the successful training of all RL policies. Intervening too frequently, such as whenever the model makes even minor mistakes, can actually reduce training efficiency by making the policy overly reliant on human assistance. To ensure fair comparisons and improve efficiency, we design an Intervention Standard Operating Procedure (ISOP) to guide the human operator’s intervention behavior. The details of ISOP are provided in the Appendix~\ref{supp:sop}.

\paragraph*{Ever-correcting Reward Classifier}
As pointed out by previous work~\cite{luo2025precise}, even when a reward classifier achieves up to $95\%$ precision after offline training, it can still produce many false-negative or false-positive examples when deployed in online environments. A common strategy is to collect additional data using a random policy and to add hand-crafted rules to improve classifier robustness, but this is labor-intensive and difficult to scale. Since a highly precise reward classifier is crucial for successful training under extremely sparse feedback, we instead introduce an ever-correcting reward classifier module during online training. The classifier is retrained whenever newly corrected labels are added to the reward buffer, whether they correct false negatives or false positives.

\paragraph*{Slow-learning Lagrange network}
Since there is no explicit upper bound on the output of the Lagrange network, the multiplier $\lambda(\bm{s})$ can easily grow to extremely large values whenever the constraint is violated, which leads to unstable training for both the Lagrange and policy networks. Following prior work on state-wise Lagrange multipliers~\cite{seo2025state}, we set the learning rate of the Lagrange network to be 0.01 times that of the other networks. Specifically, we use a learning rate of $3\times 10^{-4}$ for the $Q$, $\pi$, and $\beta$ networks, and a learning rate of $3\times 10^{-6}$ for the Lagrange network $\lambda$.

\section{Experiments}
To validate the effectiveness of proposed \ours and to investigate key factors that enable efficient real-world online training, we conduct experiments on 8 manipulation tasks. Our experiments are designed to answer the following research questions (RQ): (1) How much does \ours improve performance over SOTA reinforcement learning (RL) and imitation learning (IL) methods? (2) How robust is the policy trained by \ours compared to other methods? (3) What is the contribution of each component (i.e., the Lagrange term, RL objective, and behavior cloning (BC) objective) to the performance of \ours? (4) How do the intervention quality affect the performance of different methods? 

The remainder of this section is organized as follows. First, we describe the experimental setup, including tasks and hardware platform. Next, we introduce the baselines and evaluation metrics. Finally, we present the experimental results and analyses corresponding to each research question.

\subsection{Experimental Setup}
As shown in Fig.~\ref{fig:task}, we design 8 real-world manipulation tasks covering mixed skills, articulated-object manipulation, precise manipulation, and deformable-object handling. These tasks are challenging due to small interaction regions (e.g., spoon handle, USB port, hook), orientation constraints (e.g., articulated drawer, U-shaped groove), deformability (e.g., Chinese knot, rag), and mixed skills (e.g., pick-and-place). Different from previous works~\cite{luo2025precise,chen2025conrft} that impose strict constraints on the exploration space, we only restrict translational movements for safety considerations, while leaving the rotational space unconstrained.
Detailed task settings and parameters are provided in Appendix~\ref{sec:supp_training_details}.

For teleoperation, we use the custom human-robot collaboration system HACTS~\cite{xu2025hacts} that bilaterally synchronizes the joint states of the robot arm and the operator device in real time. Human intervention and reward correction signals are issued via a foot-pedal device. 

We evaluate all methods on two embodiments, UR5 and Franka Emika, under the same setup. For both embodiments, the state $\bm{s}$ consists of visual and proprioceptive inputs. Visual input is provided by two RGB-D cameras (a right-side view and a wrist camera): the UR5 uses Orbbec Gemini 336 and 336L units, while the Franka uses two RealSense cameras. The proprioceptive input includes the absolute end-effector pose in the base frame and the gripper state. Both embodiments are equipped with a Robotiq 2F-85 gripper and share the same action space: a 6-DoF end-effector pose increment and, for some tasks, a discrete gripper command at each timestep.
The actor process runs on an GeForce RTX 4090 GPU server, while the learner process runs on a RTX A6000 GPU server.
 
\subsection{Baselines and Metrics}
We compare \ours~with three state-of-the-art RL and IL methods as follows:
\begin{itemize}
\item \textbf{HIL-SERL}~\cite{luo2025precise}: integrates a pretrained vision backbone, sparse rewards, and the off-policy RLPD algorithm, leveraging both offline demonstrations and online data while learning from scratch.
\item \textbf{ConRFT}~\cite{chen2025conrft}: a two-stage reinforced fine-tuning framework for VLA models. In offline stage, it pairs calibrated Q-learning~\cite{nakamoto2023cal} with behavior cloning (BC) under a consistency policy. In online stage, it performs online fine-tuning with human interventions for safe and sample-efficient adaptation.

\item \textbf{HG-Dagger}~\cite{kelly2019hg}: a stable and efficient online IL algorithm that learns solely from offline demonstrations and online correction data using a BC loss.

\end{itemize}
For a fair comparison, all online methods are trained within the same actor-learner asynchronous architecture, and the models for HIL-SERL, HG-Dagger, and \ours share identical network architectures and hyperparameter settings. ConRFT is evaluated using its original JAX implementation.

All methods are evaluated on three metrics:
\begin{itemize}
\item \textbf{Success Rate}: Unlike prior work, which measures success on trajectories regardless of human intervention, we only regard an episode as successful if the robot completes the entire trajectory \emph{without any human assistance}, so as to clearly reflect the performance of the learned policy itself. Concretely, during training, the success rate is computed as the fraction of successful episodes within a fixed time window.
\item \textbf{Intervention Ratio}: human‐intervention steps divided by total steps per episode.
\item \textbf{Episode length}: the total number of steps per episode. An episode terminates either when the classifier predicts 1 or when it exceeds the maximum step limit.
\end{itemize}

\begin{figure*}[tb]
    \centering
    \subfloat{\includegraphics[width=1.0\textwidth]{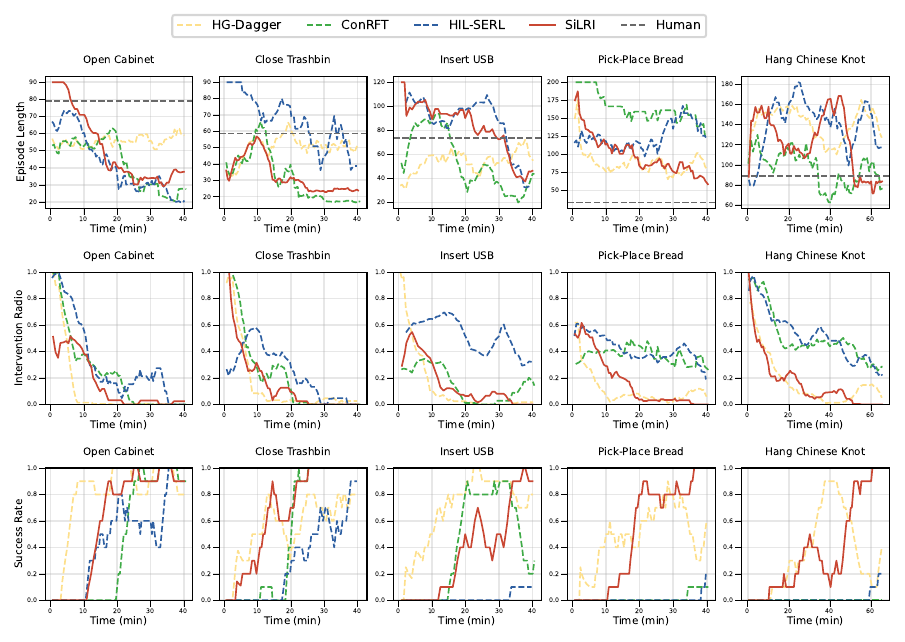}}
    \caption{\textbf{Training Curves of Episode Length, Intervention Ratio, and Success Rate.} We train four online methods on five different tasks. To ensure consistency, all methods within the same task are operated by the same human operator.}
    \label{fig:training_ratio}
\end{figure*}

\subsection{Overall Performance Comparison}
\textbf{To address RQ1}, we conduct online training experiments under the same training duration and conditions for all methods. Fig.~\ref{fig:training_ratio} presents the training curves for episode length, intervention ratio, and success rate across five tasks, while the remaining three tasks are shown in Fig.~\ref{fig:supp_training_ratio} in Appendix. We also compute the average episode length under human operation from the 20 successful offline demonstrations, denoted as “Human”.

As shown in Fig.~\ref{fig:training_ratio}, \ours achieves a high success rate earlier than the other RL methods. For example, in \textit{Open Cabinet} task, \ours reaches a success rate of around 90\% after only 15 minutes of online training, whereas ConRFT and HIL-SERL attain similar performance 10-20 minutes later, respectively. This advantage stems from the imitation objective modulated by our proposed state-wise Lagrange multiplier. Although ConRFT also incorporates an imitation loss, it uses a small coefficient (0.5) during online training. In contrast, for \ours, the average value of the Lagrange multiplier is often larger than 1 at the beginning of training (see Fig.~\ref{fig:multiplier_open_cabinet} and Fig.~\ref{fig:multiplier_close_trashbin} in Appendix, for the convergence behavior of the multiplier). This highlights the critical role of carefully balancing the RL objective and the BC objective when learning from suboptimal intervention data.

To our surprise, the online IL method HG-Dagger shows noticeably better performance in the first 10 minutes and achieves a comparable or even lower intervention ratio than other RL methods, which contrasts with the results reported in ConRFT and HIL-SERL. We attribute this difference to our expanded operation space and longer operation horizon. Prior works typically consider a much smaller workspace and shorter episodes, focusing on a single skill such as picking or insertion. In our experiments, objects may be moved by up to 30 cm along each axis, resulting in a substantially larger exploration space compared to previous setups with less than 10 cm of motion, and we impose no constraints on rotation. Consequently, RL methods require more time for exploration. Nevertheless, even under this challenging exploration regime and sparse feedback, \ours still exhibits a rapid performance warm-up, as discussed above.

Although HG-Dagger attains better performance at the beginning, the learned policy struggles to maintain consistently high performance later, especially in tasks where human operators are prone to making suboptimal interventions. For example, in the \textit{Hang Chinese Knot} task, human operators may take a long time to hang the knot on the hook due to the deformability of the Chinese knot and the small interaction area around the hook.
In this setting, \ours achieves a stable 100\% success rate after 60 minutes, whereas HG-DAGGER drops to 40\% as more suboptimal human interventions are introduced.
Across other tasks, HG-DAGGER similarly fails to maintain a 100\% success rate, while \ours remains consistently high once the policy has converged.

These results highlight the complementary yet crucial roles of RL objective and BC objective in real-world RL under suboptimal interventions. By employing self-adaptive Lagrange weights, \ours effectively marries the strengths of RL and BC, enabling robust learning on a broader class of tasks beyond precision manipulation.

Lastly, we observe a clear reduction in episode length for all RL methods, along with shorter execution times compared to HG-DAGGER. For example, in the \textit{Open Cabinet} and \textit{Close Trashbin} tasks, HG-DAGGER requires approximately 20 and 25 more steps, respectively, than \ours to complete the task. Notably, even compared to human operators, \ours demonstrates faster policy execution in tasks that require precise manipulation, such as \textit{Insert USB} and \textit{Close Trashbin}. These results indicate that RL methods, represented by \ours, are well suited for tasks demanding high execution speed and performance, potentially exceeding human-operated efficiency.

\begin{table*}[htbp]
\centering
\caption{\textbf{Robustness Experiments across 4 Robotic Manipulation Tasks.} In the original setting, the object’s initial position is the same as in training. In the disturbance setting, external perturbations are introduced to evaluate the robustness.}
\label{tab:exp-robust}
\resizebox{1.0\textwidth}{!}{%
\begin{tabular}{l|cccc|cccc} 
\hline
\multirow{2}{*}{\textbf{Task} / \textbf{Method}} 
& \multicolumn{4}{c|}{\textbf{Original}}& \multicolumn{4}{c}{\textbf{Disturbance}} \\
\cline{2-9}
 & Close Trashbin &  Push-T & Hang Chinese Knot & Pick Spoon  & Close Trashbin & Push-T & Hang Chinese Knot & Pick Spoon \\
\hline
HG-Dagger & 0.93 & 0.53 & 0.53 & \cellcolor[HTML]{E0F4FF}0.87 & 0.07 & 0.27 & 0.27 & 0.27\\
HIL-SERL & 0.93 & 0.27 & 0 & 0.47 & 0.33 & 0.2 & 0 & 0.27 \\
ConRFT & \cellcolor[HTML]{E0F4FF} 1 & 0.47 & 0 & \cellcolor[HTML]{E0F4FF}0.87 & 0.47 & 0.27 & 0 & 0.6 \\
\rowcolor[HTML]{E0F4FF}\ours & 1 & 0.67 & 1 & 0.87 & 0.93 & 0.47 & 0.53 & 0.8 \\
\hline
\end{tabular}
}
\end{table*}

\subsection{Robustness Experiments}
\textbf{To address RQ2}, we evaluate methods on four tasks where the robot is prone to making mistakes due to dynamic changes or its own execution errors. We deliberately introduce external disturbances to examine the robustness and failure recovery ability of each method. The evaluation tasks and disturbance rules are as follows:
\begin{itemize}
    \item \textbf{Close Trashbin}: Move the trashbin to 5 different poses (varying in both translation and orientation) when the robot gripper is approaching the lid. Each pose is evaluated 3 times.
    \item \textbf{Push-T}: Move T-shaped object along 5 different directions while the robot is pushing it. Each direction is evaluated 3 times.
    \item \textbf{Hang Chinese Knot}: Move the Chinese knot to five different positions just before the robot is about to pick it up. Each position is evaluated 3 times.
    \item \textbf{Pick Spoon}: Move the spoon to the left or right along the bowl rim just before the robot is about to pick it up.

\end{itemize}

For each task, we evaluate each method in 15 trials in total and check whether it can successfully recover from these disturbances within the specified time limit.

As shown in Tab.~\ref{tab:exp-robust}, we report results under both the original setting (no external disturbances) and the disturbance setting, where objects are perturbed according to the above rules. Firstly, we observe that \ours achieves high success rates on these tasks under original settings. For example, in the \textit{Hang Chinese Knot} task, only \ours reaches a 100\% success rate, whereas the other two RL methods, HIL-SERL and ConRFT, fail to complete this task fully autonomously due to the large exploration space. 

Secondly, a high success rate in the original training setting does not necessarily indicate strong robustness. For example, in the \textit{Close Trashbin} task, HG-DAGGER attains 93\% success rate in the original setting, but its performance drops to 7\% under disturbances, highlighting the tendency of online BC methods to overfit to a limited set of states. \ours also exhibits noticeable performance degradation on the \textit{Hang Chinese Knot} task, the success rate decreases by 47\% under disturbances, primarily because the policy fails to close the gripper after moving above the knot. In our implementation, the gripper is updated via a simple BC strategy, chosen due to the low control frequency and large exploration space.  
Nevertheless, for the other tasks, \ours maintains good recovery performance. For example, in \textit{Pick Spoon}, \ours achieves a success rate of 80\% under disturbances, reliably re-approaching and re-grasping the spoon after it falls into the bowl or is moved by the human.

\subsection{Ablation Studies}
Recall that in Eqn.~(\ref{eqn:actor_loss}), the policy optimization objective in \ours consists of three terms: the RL term, the Lagrange term, and the BC term. To assess the impact of each component and address \textbf{RQ3}, we conducted experiments on the \textit{Close Trashbin} task by removing these terms individually.

As shown in Fig.~\ref{fig:ablation}, removing the RL objective (green line) leads to a significant performance drop, although some success is observed around 10 minutes of online training. This highlights the necessity of guidance from the Q function, which approximates value information from interactions. Next, we remove the BC term entirely from Eqn.~(\ref{eqn:actor_loss}), optimizing only the RL objective without Lagrange multiplier normalization. In this setting, \ours w/o BC (blue line) converges similarly to the full method, but approximately 13 minutes later, highlighting the critical role of the BC term in accelerating training. 

Finally, we replace the Lagrange multiplier with a constant value of 0.5. The policy (yellow line) achieves a rapid warm-up similar to \ours, completing the task autonomously without human intervention after 10 minutes. However, the constant coefficient can harm long-term performance, resulting in a performance drop when additional suboptimal human interventions are introduced during online training.

\begin{figure}
    \centering
    \includegraphics[width=0.8\columnwidth]{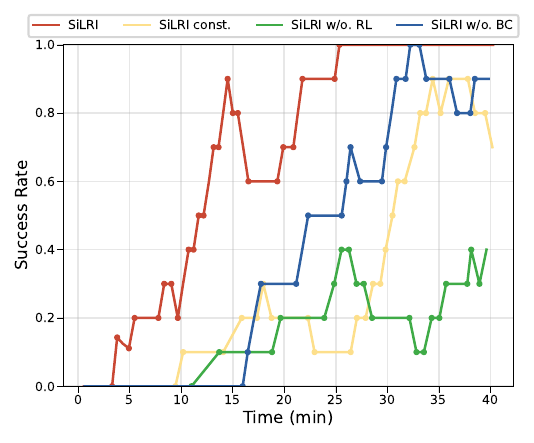}
    \caption{\textbf{Ablation Experiments in Close Trashbin.} 
    \label{fig:ablation}}
\end{figure}

\begin{figure}
    \centering
    \subfloat{\includegraphics[width=0.8\columnwidth]{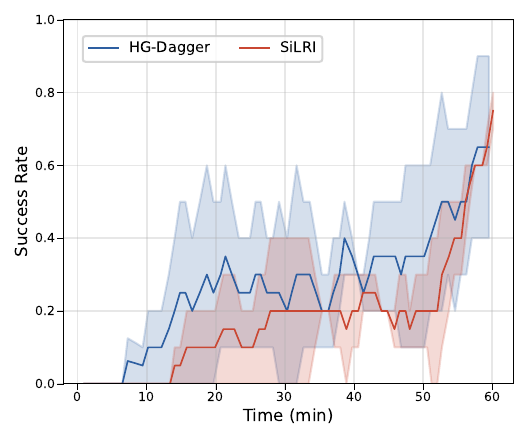}}
    \caption{\textbf{Training curves of different methods in the Push-T task under two human operators (skilled and unskilled).}}
    \label{fig:intervention}
\end{figure}

\subsection{Investigations on Data Quality}
To investigate how intervention quality and the exploration space affect training performance, we conduct experiments under controlled settings. We consider two human operators: one skilled and familiar with the teleoperation system and intervention timing, and the other unskilled and prone to making mistakes during intervention. 

We quantitatively measure data quality using the number of steps per successful episode, where a shorter episode length corresponds to more accurate and higher-quality intervention data. In the Push-T task, the skilled operator requires 99 steps on average to complete an episode, whereas the unskilled operator requires 125 steps.

As shown in Fig.~\ref{fig:intervention}, we plot the training curves of HG-Dagger and \ours on the \emph{Push-T} task, where the mean and confidence interval are computed from two training processes conducted by the skilled and unskilled human operators. During 50--60 minutes of online training, \ours exhibits more stable performance, maintaining a success rate of around 74\%. In contrast, HG-Dagger is more sensitive to intervention quality: although it also reaches a relatively high average success rate of 63\% at 60 minutes, its lower confidence bound can drop to 40\%. This demonstrates that, compared with the online IL method HG-Dagger, the RL-based \ours is less dependent on near-optimal interventions and remains effective across human operators with varying skill levels.

\section{Conclusion and Future Work}
In this work, we propose a state-wise Lagrangian reinforcement learning (RL) algorithm from suboptimal interventions, for real-world robot manipulation training. Observing the fact that human operators have different confidence level and manipulation skill over different states, a state-dependent constraint is added to the RL objective to automatically adjust the distance between human policy and learned policy. Building on a human-as-copilot teleoperation system, we evaluate our method with other state-of-the-art online RL and imitation learning methods on 8 manipulation tasks on two embodiments. Experimental results show the efficiency of \ours to utilize the suboptimal interventions at the beginning of training and converge to a high success rate at the end. Other ablation studies and investigation experiments also conducted to learn the advantage of \ours.

During real-world experiments, we identify several limitations of current online training methods, including \ours. 

First, the training time required for online methods to reach human-comparable performance remains long. Even for an unskilled teleoperation operator, the Hang Chinese Knot task can typically be handled after only a few minutes of practice. In contrast, online learning methods require at least 60 minutes of training to achieve a workable solution. In future work, we plan to incorporate vision-language-action models with pre-existing manipulation skills to reduce the required online fine-tuning time.

Second, there remains substantial room for improvement in robustness. While \ours demonstrates strong performance in static settings, its robustness degrades in dynamic environments, particularly for challenging tasks with large exploration spaces. Beyond using intervention data directly, richer intervention signals could be leveraged to implicitly constrain the exploration space and improve robustness.

Despite these limitations, we believe that online reinforcement learning remains a promising paradigm for achieving superhuman intelligence, given its efficiency in daily manipulation tasks and its potential to surpass human performance in precise manipulation scenarios.

\bibliographystyle{unsrt}
\bibliography{ref}

\clearpage
\newpage

\section{Task Training Details}\label{sec:supp_training_details}
\subsection{Details on Parameter Settings}
Here, we provide detailed descriptions of the overall training parameters and task settings. 

The task-specific settings are summarized in Tab.~\ref{tab:task_settings}. The space constraint is crucial for efficient training of HIL-SERL and ConRFT. This is because their Q-functions rely on self-collected experience to obtain accurate value estimates; simply intervening when the robot drifts far from the target is insufficient to provide informative data for these methods. In contrast, for \ours and HG-Dagger, the boundary primarily serves as a safety mechanism, since the imitation loss already guides the policy toward human behavior.

Tab.~\ref{tab:training_settings} lists the key hyperparameters used in our experiments. The same training hyperparameters are applied across all tasks. To ensure a fair comparison, HG-Dagger, HIL-SERL, and \ours share identical parameter settings, while ConRFT follows its officially released hyperparameters and network architectures.

\subsection{Details on Lagrange Multipliers}
We also provide the training curves of the Lagrange multipliers in Fig.~\ref{fig:multiplier_open_cabinet} and Fig.~\ref{fig:multiplier_close_trashbin} to illustrate the self-adaptive trade-off between the RL objective and the BC objective in \ours. Specifically, we plot the batch-mean output of the Lagrange multiplier network during training. As shown in Fig.~\ref{fig:multiplier_open_cabinet}, at the beginning of training, the large discrepancy between the behavior policy $\beta$ and the actor policy $\pi$ causes the Lagrange multiplier to quickly rise above 1, so that updates to $\pi$ are mainly driven by the BC loss. As training progresses and $\pi$ becomes closer to $\beta$, the Lagrange multiplier gradually decreases and stabilizes around 0.3, shifting the emphasis toward the RL objective.

\begin{table*}[t]
\centering
\caption{\textbf{Task Setting Details}. All tasks share a 7-dimensional action space: 6-DoF end-effector pose increments and discrete gripper actions (open or close). When the gripper is fixed for a task, the last dimension remains a constant value; when the gripper is not fixed, the last dimension is generated by a separate policy network.}
\resizebox{1.0\textwidth}{!}{
\begin{tabular}{lccccc}
\toprule
\textbf{Task/Parameter} & \textbf{Fix Gripper}& \textbf{Max episode length} & \textbf{Train time} & \textbf{Space constraint} & \textbf{Embodiment} \\
\midrule
Close Trashbin & True & 90 & 40 min  &$0.47 \leq x\leq  0.68$ m, $-0.12 \leq y\leq  0.10 $ m, $0.29 \leq z\leq  0.48$ m & Franka\\
Open Cabinet & True & 90 & 40 min  &$0.37 \leq x\leq  0.52$ m, $-0.11 \leq y\leq  0.13 $ m, $0.39 \leq z\leq  0.49$ m & Franka\\
Insert USB & True & 120 & 40 min  &$0.60 \leq x\leq  0.67$ m, $-0.14 \leq y\leq  -0.03 $ m, $0.38 \leq z\leq  0.51$ m & Franka\\
Push-T & True & 120 & 60 min  &$0.45 \leq x\leq  0.69$ m, $-0.05 \leq y\leq  0.33 $ m, $0.16 \leq z\leq  0.28$ m & Franka\\
Hang Chinese Knot & False & 200 & 65 min  &$0.47 \leq x\leq  0.74$ m, $-0.16 \leq y\leq  0.13 $ m, $0.17 \leq z\leq  0.45$ m & Franka\\
Pick-place Bread  & False & 200 & 40 min & $-0.65 \leq x\leq  -0.27$ m, $-0.42 \leq y\leq  -0.06 $ m, $0.19 \leq z\leq  0.31$ m & UR\\
Fold Rag  & False & 100 & 40 min &$-0.55 \leq x\leq  -0.37$ m, $-0.62 \leq y\leq  -0.42 $ m, $0.29 \leq z\leq  0.41$ m & UR\\
Pick-up Spoon  & False & 100 & 60 min  &$-0.51 \leq x\leq  -0.35$ m, $-0.43 \leq y\leq  -0.27 $ m, $0.22 \leq z\leq  0.31$ m & UR\\

\bottomrule
\end{tabular}
}
\label{tab:task_settings}
\end{table*}

\begin{table}[t]
\centering
\caption{\textbf{Training Details}.}
\begin{tabular}{ll}
\toprule
\textbf{Parameter} & \textbf{Value}\\
\midrule
Online Buffer Batch Size & 128\\
Offline Buffer Batch Size & 128\\
Actor Learning Rate  & 0.0003\\
Critic Learning Rate  & 0.0003\\
Expert Learning Rate  & 0.0003\\
Lagrange Learning Rate  & 0.000003\\
Optimizer & Adam\\
Discount Factor & 0.97\\
Initial Offline Demonstrations & 20\\
$\kappa$ & 6 \\
\bottomrule
\end{tabular}
\label{tab:training_settings}
\end{table}

\subsection{Details on Metric Computation}
Beyond the algorithm itself, real-world online training results are influenced by many external factors. As a result, the training curves can be noisy and exhibit substantial fluctuations, especially for binary metrics such as the success rate and intervention ratio. To more clearly show performance trends over the course of training without loss of authenticity, we apply the following smoothing function to post-process the Success Rate, Intervention Ratio, and Episode Length metrics.

We use the sliding window average method to process the index sequence, and calculate the arithmetic mean of the data in the specified window to smooth the noise fluctuation, while retaining the time series characteristics of the training trend. Specifically, the sliding window is composed of the data of each episode and its first nine episodes, and the average success rate in the window is calculated point by point.

\section{Theoretical Analysis}
\label{sec:proof}

This section shows that, under standard Gaussian modeling assumptions, the
state-wise KL constraint in Eqn.~(\ref{eqn:original_objective}) implies a
state-wise constraint on the distance between policy means, as used in
Eqn.~(\ref{eqn:objective}).

\begin{assumption}[Isotropic Gaussian Policies]
\label{ass:gaussian}
For any state $\bm{s}$, the human reference policy is modeled as an isotropic Gaussian
\begin{equation}
\beta(\cdot \mid \bm{s}) = \mathcal{N}\!\big(\bm{\mu}_\beta(\bm{s}), \ \sigma_\beta(\bm{s})^2 \bm{I}\big),
\end{equation}
and the learned policy is modeled as an isotropic Gaussian with \emph{fixed} standard deviation
\begin{equation}
\pi(\cdot \mid \bm{s}) = \mathcal{N}\!\big(\bm{\mu}_\pi(\bm{s}), \ \sigma_\pi^2 \bm{I}\big),
\end{equation}
where $\sigma_\pi>0$ is a constant, $\bm{I}$ is the $d\times d$ identity matrix, and $d$ denotes the action dimension.
\end{assumption}

\begin{assumption}[Bounded Human Uncertainty]
\label{ass:bounded_sigma}
There exist constants $0<\underline{\sigma}\le \overline{\sigma}<\infty$ such that
$\sigma_\beta(\bm{s})\in[\underline{\sigma},\overline{\sigma}]$ for all $\bm{s}$.
\end{assumption}

\begin{lemma}[KL Divergence Between Gaussians]
\label{lem:gaussian_kl}
For two non-degenerate Gaussians $\mathcal{N}(\bm{m}_1,\Sigma_1)$ and $\mathcal{N}(\bm{m}_2,\Sigma_2)$, the KL divergence admits the closed form
\begin{equation}
\begin{aligned}
\mathrm{D}_{\mathrm{KL}}&\left(\mathcal{N}(\bm{m}_1,\Sigma_1)\Vert\mathcal{N}(\bm{m}_2,\Sigma_2)\right) \\
&=
\frac{1}{2}\Big(
\mathrm{tr}(\Sigma_2^{-1}\Sigma_1)
+
(\bm{m}_2-\bm{m}_1)^\top \Sigma_2^{-1}(\bm{m}_2-\bm{m}_1)\\
&\quad-d
+
\ln\frac{\det \Sigma_2}{\det \Sigma_1}
\Big).
\end{aligned}
\label{eq:gaussian_kl}
\end{equation}
\end{lemma}

Applying Lemma~\ref{lem:gaussian_kl} under Assumption~\ref{ass:gaussian} with
$\Sigma_\pi=\sigma_\pi^2\bm{I}$ and $\Sigma_\beta(\bm{s})=\sigma_\beta(\bm{s})^2\bm{I}$ yields, for any $\bm{s}$,
\begin{align}
\mathrm{D}_{\mathrm{KL}}&\big(\pi(\cdot\mid \bm{s}) \Vert \beta(\cdot\mid \bm{s})\big)\\
&=
\frac{1}{2}\Big(
d\frac{\sigma_\pi^2}{\sigma_\beta(\bm{s})^2}
+
\frac{\|\bm{\mu}_\pi(\bm{s})-\bm{\mu}_\beta(\bm{s})\|_2^2}{\sigma_\beta(\bm{s})^2}
-d
+
d\ln\frac{\sigma_\beta(\bm{s})^2}{\sigma_\pi^2}
\Big) \nonumber\\
&=
\frac{1}{2}\cdot
\frac{\|\Delta\bm{\mu}(\bm{s})\|_2^2}{\sigma_\beta(\bm{s})^2}
+
\frac{1}{2}c\!\left(\sigma_\beta(\bm{s})\right),
\label{eq:kl_split}
\end{align}
where $\Delta\bm{\mu}(\bm{s}) \triangleq \bm{\mu}_\pi(\bm{s})-\bm{\mu}_\beta(\bm{s})$ and
\begin{equation}
c(\sigma)\triangleq
d\Big(
\frac{\sigma_\pi^2}{\sigma^2}-1+\ln\frac{\sigma^2}{\sigma_\pi^2}
\Big).
\label{eq:c_sigma}
\end{equation}
Moreover, since $\ln x \le x-1$ for any $x>0$, we have $c(\sigma)\ge 0$ for all $\sigma>0$.


\begin{theorem}[From a KL Constraint to a Mean-Distance Constraint]
\label{thm:mean_bound}
Suppose Assumptions~\ref{ass:gaussian}--\ref{ass:bounded_sigma} hold and
\begin{equation}
\mathrm{D}_{\mathrm{KL}}\big(\pi(\cdot\mid \bm{s}) \Vert \beta(\cdot\mid \bm{s})\big)\le \epsilon,
\qquad \forall \bm{s}.
\label{eq:kl_constraint_app}
\end{equation}
Then, for all $\bm{s}$,
\begin{equation}
\frac{\|\Delta\bm{\mu}(\bm{s})\|_2^2}{\sigma_\beta(\bm{s})^2}
\le
2\epsilon - c\!\left(\sigma_\beta(\bm{s})\right),
\label{eq:mean_bound_state}
\end{equation}
where $\Delta\bm{\mu}(\bm{s})\triangleq \bm{\mu}_\pi(\bm{s})-\bm{\mu}_\beta(\bm{s})$ and
$c(\sigma)$ is defined in Eqn.~(\ref{eq:c_sigma}).
Moreover, letting
\begin{equation}
c_{\max}\triangleq \max_{\sigma\in[\underline{\sigma},\overline{\sigma}]} c(\sigma) < \infty,
\label{eq:cmax}
\end{equation}
we have
\begin{equation}
\|\Delta\bm{\mu}(\bm{s})\|_2^2
\le
\kappa\, \sigma_\beta(\bm{s})^2,
\qquad \forall \bm{s},
\label{eq:mean_bound_final}
\end{equation}
with $\kappa \triangleq 2\epsilon - c_{\max}$.
\end{theorem}

\begin{proof}
By Eqn.~(\ref{eq:kl_split}), the KL constraint in Eqn.~(\ref{eq:kl_constraint_app}) implies
\(
\frac{1}{2}\frac{\|\Delta\bm{\mu}(\bm{s})\|_2^2}{\sigma_\beta(\bm{s})^2}
+\frac{1}{2}c(\sigma_\beta(\bm{s}))\le \epsilon
\)
for all $\bm{s}$, which rearranges to Eqn.~(\ref{eq:mean_bound_state}).
Under Assumption~\ref{ass:bounded_sigma}, $c(\sigma_\beta(\bm{s}))\le c_{\max}$, hence
\(2\epsilon-c(\sigma_\beta(\bm{s}))\ge 2\epsilon-c_{\max}=\kappa\).
Multiplying Eqn.~(\ref{eq:mean_bound_state}) by $\sigma_\beta(\bm{s})^2$ gives Eqn.~(\ref{eq:mean_bound_final}).
\end{proof}

\begin{figure}[tb]
    \centering
    \subfloat{\includegraphics[width=0.8\columnwidth]{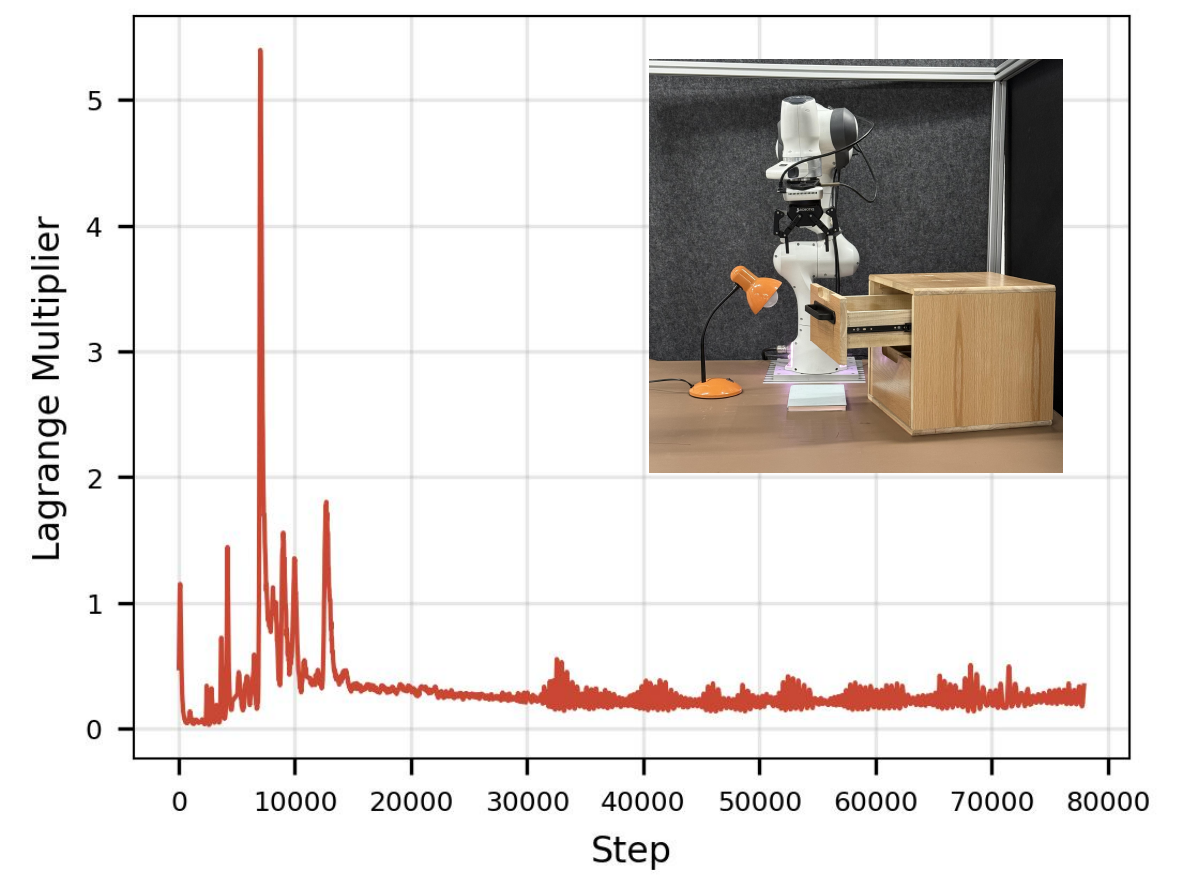}}
    \caption{\textbf{Lagrange Multiplier Training Curves in Open Cabinet}}
    \label{fig:multiplier_open_cabinet}
\end{figure}

\begin{figure}[tb]
    \centering
    \subfloat{\includegraphics[width=0.8\columnwidth]{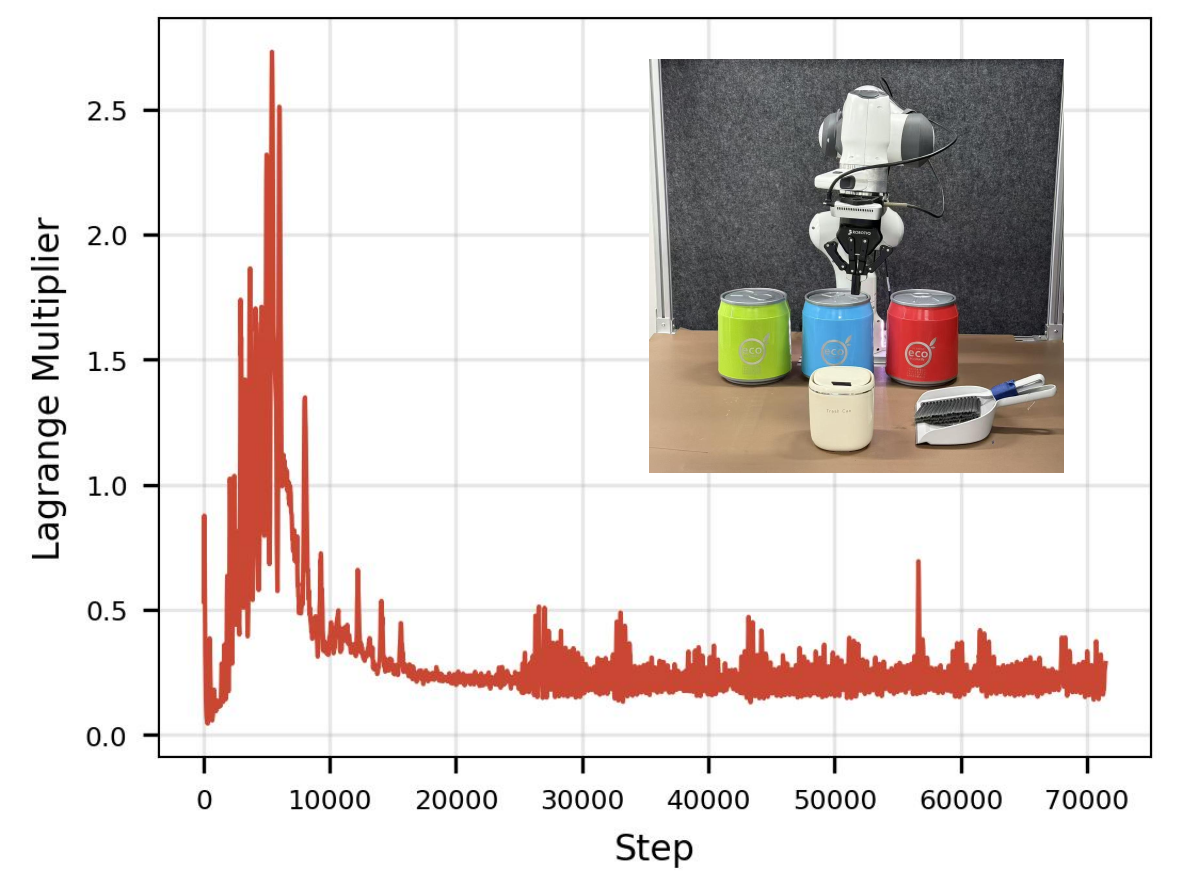}}
    \caption{\textbf{Lagrange Multiplier Training Curves in Close Trashbin}}
    \label{fig:multiplier_close_trashbin}
\end{figure}

\begin{figure*}[tb]
    \centering
    \subfloat{\includegraphics[width=0.7\textwidth]{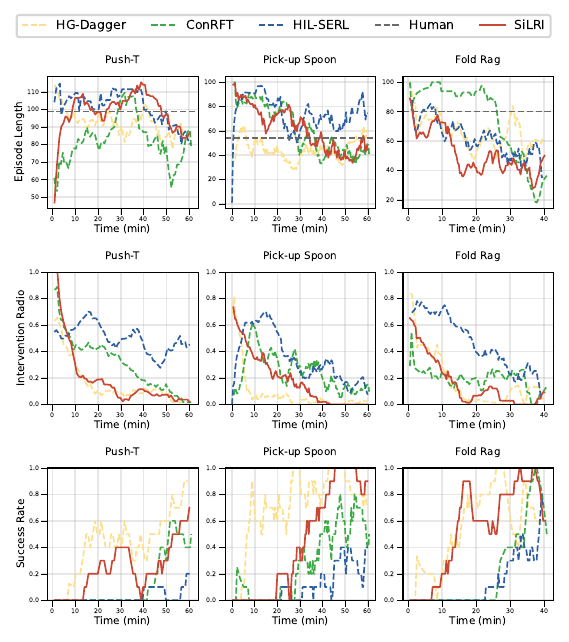}}
    \caption{\textbf{Training Curves of Episode Length, Intervention Ratio, and Success Rate.} We train four online methods on three different tasks. To ensure consistency, all online training is conducted by the same human operator.}
    \label{fig:supp_training_ratio}
\end{figure*}

\section{Intervention Standard Operating Procedure}\label{supp:sop}
During our real-world experiments, we observed that the timing of human interventions significantly affects convergence performance, particularly for ConRFT and HIL-SERL. This is primarily due to the mismatch between slow physical execution and a large exploration space, which limits the robot's ability to explore efficiently.

In some high-uncertainty states, even with human intervention, the robot must still explore alternative actions and their outcomes. In such cases, allowing the robot to explore independently (constrained by safety boundaries) is often more effective than frequent human guidance.

To address this challenge and ensure fair comparisons across methods, we developed a standardized Intervention Standard Operating Procedure (SOP) as follows:

\begin{enumerate}
    \item \textbf{Initial full interventions.} Fully intervene for the first 3-5 episodes, completing the task each time to provide stable initial guidance.

    \item \textbf{No-intervention exploration after severe deviation.} Once a \textbf{severe deviation} is observed (e.g., when the robot should move toward the trash bin but the end-effector instead explores in an unrelated direction), \textbf{do not intervene}. Let the robot explore until the episode \textbf{times out}.

    \item \textbf{Early intervention after repeated severe deviations.} If step (2) is repeated and a severe deviation is observed \textbf{twice} at the \textbf{same location/state}, then in the next episode, \textbf{intervene earlier} at the earliest point where the deviation first occurs.

    \item \textbf{Fallback to full demonstration after repeated failures.} If the robot fails to complete the task for \textbf{five} consecutive episodes under the above protocol, then in the next episode, \textbf{fully intervene from the beginning} and complete the task, providing a complete human demonstration.
\end{enumerate}

While it is impractical for human operators to strictly follow the SOP during execution, it still serves as a valuable guideline for efficient online training.

\end{document}